%% file: paper.tex
\documentclass[accepted]{uai2023} 

\usepackage[american]{babel}

\usepackage{natbib} 
\usepackage{mathtools} 
\usepackage{booktabs} 
\usepackage{tikz} 

\usepackage{algorithm}
\usepackage[algo2e, ruled, vlined]{algorithm2e}
\usepackage{algorithmic}
\usepackage{amsthm}
\usepackage{amssymb}
\usepackage{caption}
\usepackage[capitalize,noabbrev]{cleveref}
\usepackage{colortbl}
\usepackage{flushend}
\usepackage{subcaption}

\theoremstyle{plain}
\newtheorem{theorem}{Theorem}[section]

\newtheorem{lemma}[theorem]{Lemma}

\theoremstyle{definition}
\newtheorem{definition}[theorem]{Definition}

\theoremstyle{remark}

\newcommand\Tau{\mathcal{T}}

\input{math_commands.tex}

\title{Vacant Holes for Unsupervised Detection of the Outliers in \\ Compact Latent Representation}

\author[1]{\href{mailto:<m.glazunov@tudelft.nl>?Subject=Your UAI 2023 paper}{Misha~Glazunov}{}}
\author[2]{Apostolis~Zarras}

\affil[1]{%
    Delft University of Technology\\
    the Netherlands
}
\affil[2]{%
    University of Piraeus\\
    Greece
}
  
\begin{document}
\maketitle

\begin{abstract}
Detection of the outliers is pivotal for any machine learning model deployed and operated in real-world. It is essential for the Deep Neural Networks that were shown to be overconfident with such inputs. Moreover, even deep generative models that allow estimation of the probability density of the input fail in achieving this task. In this work, we concentrate on the specific type of these models: Variational Autoencoders (VAEs). First, we unveil a significant theoretical flaw in the assumption of the classical VAE model. Second, we enforce an accommodating topological property to the image of the deep neural mapping to the latent space: compactness to alleviate the flaw and obtain the means to provably bound the image within the determined limits by squeezing both inliers and outliers together. We enforce compactness using two approaches: $(i)$~Alexandroff extension and  $(ii)$~fixed Lipschitz continuity constant on the mapping of the encoder of the VAEs. Finally and most importantly, we discover that the anomalous inputs predominantly tend to land on the vacant latent holes within the compact space, enabling their successful identification. For that reason, we introduce a specifically devised score for hole detection and evaluate the solution against several baseline benchmarks achieving promising results.
\end{abstract}

\section{Introduction}

Deep Generative Models (DGMs) allow for estimating the probability density of the input. This capability may appear tempting to utilize in the tasks of the detection of the outliers by casting all of the inputs that lie Out-of-Distribution (OoD) with the low density as anomalous. Nevertheless, empirical evidence shows that DGMs may sometimes be overconfident in their density estimation over OoDs~\citep{nalisnick2018deep}. Overconfidence is observed in all types of DGMs, including autoregressive models~\citep{pixelcnn}, normalizing flows~\citep{dinh2017density}, and VAEs~\citep{kingma2013autoencoding, rezende2014stochastic}. This fact may appear especially intriguing, considering the difference in the techniques used for density estimation among these three distinct modeling approaches. However, from the theoretical perspective, there is nothing peculiar in such performance. It can be easily demonstrated that it is possible to learn an invertible reparametrization of the actual density of the data in a way that assigns an arbitrary density to each point in the new representation even in the models with perfect densities and in a low-dimensional setting~\citep{Lan2020PerfectDM}. It means that the outlier detection is infeasible while relying only on the arbitrary learned probability density.



There are several alternative approaches aiming at tackling this issue that can be coarsely classified into one of the following categories: $(i)$~methods that augment the input data by outliers~\citep{hendrycks2018deep, ren2019likelihood}, $(ii)$~ensemble-based methods~\citep{daxberger2019bayesian, glazunov2022do, choi2019waic}, $(iii)$~methods that introduce new scores~\citep{nalisnick2019detecting, Serr2020GENERATIVEM}, $(iv)$~methods based on the model modification~\citep{HernndezLobato2016ImportanceWA, Schirrmeister2020UnderstandingAD}, $(v)$~and methods that involve retraining of the models~\citep{xiao2020likelihood}.

In this work, we refrain from augmenting data with outliers during training since it is not always feasible; we do not retrain the model to check every input as it is time-consuming, and due to the same reason, we do not apply ensemble-based methods. Instead, we utilize a model modification by introducing a new score. Specifically, we address the outlier detection from the perspective of general topology. Namely, we consider the property of compactness of the mapped image in the latent space. This property satisfies the necessary condition for the modeling assumption of a classical VAE from the viewpoint of the Universal Approximation Theorem (UAT)~\citep{cybenko, hornik, pinkus}. First, we implement compactification using the Alexandroff extension of a flat subspace to a hypersphere. Second, we utilize a related topological property: bounded continuity. It equips us with two additional valuable tools. In particular, it lets enforce the Lipschitz-continuity constraints on the mappings used in the model. These constraints, in turn, permit both to establish the compactness of the mapped image and simultaneously control its boundaries in the case of the flat latent space. In addition, it helps to identify if the continuity holes in the latent prior play a significant role in the outlier detection during the ablation study.

Constraining the mapped image of the encoder may at first sound counterintuitive since the common choice of a prior over the latent is used to be the standard normal distribution with the infinite support that explicitly implies that outliers should be placed in some different location, distinctly separated from the inliers. It includes the low-dimensional cases where such inputs are placed in the tails far from the mode and the high-dimensional cases where the outliers are located outside the typical set. However, as we already indicated, there is no guarantee for such behavior even in perfect density models since any density function can be manipulated by an arbitrary choice of representation. Since there is no control over the mapped compact in the latent space, the choice of the bounds of the learned factors of variations of the VAE is basically arbitrary. In some situations, it can be the case that the outliers are indeed placed far from the inliers, which gives an excellent separation based only on the density values; however, in other situations, the outliers and inliers may overlap, which in some cases results in the overconfidence of the model. Hence the purposeful control over the compactness of the mapped image enforces the model to bind the learned factors of variations for \emph{any} input within the predefined limits. If these limits are chosen in such a way that enforces the model to squeeze all of its inputs in the properly bounded space, then the model would have no other choice than to map the outliers somewhere \emph{within} the same space that is used for the inliers in the latent representation. Experimental evidence shows that when the model is confronted with such tight condensing, it tends to place the outliers into the vacant latent continuity holes allowing their successful detection.

In summary, we make the following main contributions:

\begin{itemize}
\item We reveal the persistent theoretical flaw in the modeling assumption of VAEs.
\item We mitigate this shortcoming by enforcing controlled compactness of the latent space.
\item By bounding the image of the encoder, we discover that the outliers tend to gravitate toward the vacant latent holes and devise an appropriate score for their detection.
\item We empirically evaluate the suggested approach based on several datasets.
\end{itemize}

\section{Background}

\subsection{Notation}

We use nonbold $x$'s to denote elements of general topological spaces, including the ones equipped with the appropriate metric. In the case of the normed vector spaces and random vectors within such spaces, we adhere to traditional usage in the literature, namely $\rvx$. When it comes to the particular elements comprising the random vector, we utilize $\ervx$. The spaces are denoted as a pair $(\mathcal{X}, \Tau)$ for topological spaces with the corresponding topology $\Tau$. In the specific case of metric spaces, we indicate the appropriate metric $d$ that induces topology: $(\mathcal{X}, d_{\mathcal{X}})$.

\subsection{VAEs}

VAE represents a DGM that allows to get an approximate value of the density of the input $\rvx$. It is based on the optimization of the evidence lower bound (ELBO), that provides joint optimization w.r.t variational parameters $\boldsymbol{\phi}$ of the encoder responsible for variational approximation of the posterior $q_{\boldsymbol{\phi}}$ over the latent variable $\rvz$, and the generative parameters $\boldsymbol{\theta}$ of the decoder responsible for the parameterization of the likelihood $p_{\boldsymbol{\theta}}(\rvx|\rvz)$:
%
\begin{equation}
\label{eq:elbo_z}
\mathcal{L_{\boldsymbol{\theta}, \boldsymbol{\phi}}}(\rvx) = \E_{q_{\boldsymbol{\phi}}(\rvz|\rvx)} \Big[\log p_{\boldsymbol{\theta}}(\rvx|\rvz) - \log \frac{q_{\boldsymbol{\phi}}(\rvz|\rvx)}{p_{\boldsymbol{\theta}}(\rvz)}\Big]
\end{equation}

This equation involves a data likelihood term (used for generative purposes) and a regularization term (the KL divergence between the variational family $q_{\phi}(\rvz|\rvx)$ and the prior distribution over the latent variables).

The final estimation of the marginal likelihood is done using importance sampling. Backpropagation through the random variable $\rvz$ is performed utilizing the standard reparameterization trick~\citep{kingma2013autoencoding}.

\subsection{Compactness}

A topological space is compact or, equivalently, possesses a compactness property if every of its open cover has a finite subcover. In the case of the Eucledian spaces, the following specific result exists.

\begin{theorem}[Heine-Borel]
\label{thm:heineborel}
Let $K \subset \reals^n$ then $K$ is compact if and only if $K$ is closed and bounded.
\end{theorem}

Compactification is the process of turning a topological space into a compact one.

\begin{definition}
\label{def:compactification}
 Let $(\mathcal{X}, \Tau)$  be a topological space and let $(\mathcal{X}^*, \Tau^*)$ be a compact topological space s.t. $\mathcal{X}$ is homeomorphic to a dense subspace of $\mathcal{X}^*$. Then $(\mathcal{X}^*, \Tau^*)$ is called a compactification of $(\mathcal{X}, \Tau)$. Thus, a compact space $(\mathcal{X}^*, \Tau^*)$ is a compactification of a space $(\mathcal{X}, \Tau)$ if and only if there exists a mapping $f$ of $\mathcal{X}$ into $\mathcal{X}^*$ s.t. $f$ is homeomorphism of $\mathcal{X}$ onto the subspace $f(\mathcal{X})$ of $\mathcal{X}^*$ and $f(\mathcal{X})$ is dense in $\mathcal{X}^*$.
\end{definition}

An illustrative example of a frequently used compactification is an extension of $\reals$ to $\reals \cup \{-\infty, +\infty\}$.

Besides, there is a specific type of compactification by adjoining only one point: the Alexandroff extension.

\begin{definition}
\label{def:alexandroff}
 Let $(\mathcal{X}, \Tau)$ be a topological space and let $\infty$ be an object not belonging to $\mathcal{X}$.
Let $\mathcal{X}^* = \mathcal{X} \cup \infty$ and let a topology $\Tau^*$ on $\mathcal{X}^*$ defined as follows: $ \mathcal{\Tau}^* = \mathcal{\Tau} \cup \{ V \subset \mathcal{X}^* : \infty \in V \textrm{and} \  \mathcal{X} \setminus V \ \textrm{is closed and compact in} \  \mathcal{X} \}$. Then $(\mathcal{X}^*, \Tau^*)$ is the Alexandroff extension of $(\mathcal{X}, \Tau)$.



\end{definition}

An intuitive example of the Alexandroff extension is the inverse stereographic projection from the Euclidean plane to the sphere with the addition of a point at infinity.

\subsection{Lipschitz continuity}

\begin{definition}
\label{def:Lip}
A map $f: \mathcal{X} \to \mathcal{Y}$, where $(\mathcal{X}, d_{\mathcal{X}})$ and $(\mathcal{Y}, d_{\mathcal{Y}})$ are metric spaces with the corresponding metrics $d_{\mathcal{X}}$ and $d_{\mathcal{Y}}$, is called Lipschitz continuous if for any $x_1,x_2\in \mathcal{X}$, there exists a constant $M \in \mathbb{R^{+}}$ such that:
\begin{equation}
d_{\mathcal{Y}}(f(x_1), f(x_2)) \leq M d_{\mathcal{X}}(x_1, x_2)    
\end{equation}

\end{definition}

$M$ is called a Lipschitz constant. In this work, we refer to the Lipschitz constant as the smallest possible $M$. A mapping with such a constant is called an $M$-Lipschitz map. If not explicitly indicated otherwise, we let $\mathcal{X} = \reals^n$ and $\mathcal{Y} = \reals^m$.

Recall that widely used activation functions such as sigmoid, tanh, and ReLU~\citep{jarrett} are already generally scaled to be 1-Lipschitz. Hence, due to the composition property of the Lipschitz mappings, the first intuitive attempt to enforce the desirable Lipschitz property on the mapping would be to constrain the operator norm of the weights of each layer of the Deep Neural Network (DNN)~\citep{yoshida, cisse}. However, it was proven that such an approach could not approximate even a simple absolute value function~\citep{huster}. To tackle the issue, \cite{anil} observed the critical component that influences the expressive power of any DNN, namely, the gradient-preserving property of its transformations. Therefore, they introduced the appropriate linear transformations and the 1-Lipschitz activation function, GroupSort, both of which are gradient preserving. They provably allow setting a Lipschitz constant on a DNN mapping. Moreover, DNNs utilizing them represent universal approximators of any Lipschitz mapping.

\subsubsection{Latent Holes}\label{latent_holes_duality}
\cite{falorsi2018} introduced a score for detecting continuity holes in the latent space based on the ratio of the distances between two nearly located points in the input space and the distances of their corresponding latent codes: 
\begin{equation}
\mathcal{F}_{Lip}= d_{\mathcal{Y}}(f(x_1), f(x_2)) / d_{\mathcal{X}}(x_1, x_2)
\end{equation}

\cite{xu2020} discovered that there exist vacant regions of low density in the aggregated posterior where prior assigns a relatively high density. They suggested detecting these regions by estimating the negative log-likelihood of the manipulated reference latent codes under the aggregated posterior:
\begin{equation}
\mathcal{F}_{Agg} = - \log p(\mathbf{z} \pm \mathbf{\epsilon})
\end{equation}

where $\epsilon$ represents a magnitude of manipulation. It was demonstrated by~\cite{ruizhe2021} that both of these scores are connected despite the different motivation meaning that if the hole is detected by the score $\mathcal{F}_{Agg}$ then it will be also detected by $\mathcal{F}_{Lip}$.

\subsection{Universal Approximation Theorem}

The theoretical underpinning of DNNs is rooted in the results obtained in the approximation theory that is commonly referred to as the universal approximation theorem ~\citep{cybenko, hornik, pinkus}.

\begin{theorem}
\label{thm:uat}
Let $C(\mathcal{X}, \mathcal{Y})$ denote the set of all continuous mappings from $\mathcal{X}$ to $\mathcal{Y}$. Let $\sigma \in C(\reals, \reals)$ represent an element-wise activation function. Then let $\NNgen$ represent the class of feedforward neural networks with activation function $\sigma$, with $n$ neurons in the input layer, $m$ neurons in the output layer, and one hidden layer with an arbitrary number of neurons. Let $K \subseteq \reals^n$ be \textbf{compact}. Then $\sigma$ is nonpolynomial if and only if $\NNgen$ is dense in $C(K, \mathcal{Y})$.
\end{theorem}

The activation functions currently used in DNNs are non-polynomial, so they fulfill the main requirement of the theorem. However, we deliberately emphasize that the results of the Universal Approximation Theorem (UAT) apply only in the cases when the input of the neural network is a compact set that is often overlooked.

\section{Related Work}
\label{sec:relatedwork}

\medskip \textbf{New Scores-Based Methods.} \cite{nalisnick2019detecting} conjecture that considering the high dimensionality of inputs, the over-confidence of DGMs may be because in-distribution images lie in the typical set as opposed to the tested OoDs that concentrate in the high-density region. They introduce the test for typicality that treats all input sequences as inliers if their entropy is close to the model's entropy.

Since the likelihood of generative models is biased by the complexity of the inputs, \cite{Serr2020GENERATIVEM} propose to offset this bias by a factor that measures the input complexity and use the length of lossless compression of the image as the complexity factor, which is used to determine OoD. However, they do not evaluate their method on VAEs.

\medskip \textbf{Ensemble-Based Methods.} \cite{choi2019waic} use an ensemble of independently trained DGMs that allow to get the density value and score them against the WAIC.
 
\medskip \textbf{Bayesian DGMs}. Although the BDGMs represent a single model, the Bayesian inference over model parameters allows building ensembles on the fly. The theoretical justification for the Bayesian VAE has been first laid out by~\cite{kingma2013autoencoding}. Several works are dedicated to OoD detection using Bayesian inference~\citep{daxberger2019bayesian, glazunov2022do}. They introduced new scores, such as the disagreement score and entropy. Both are based on the discrepancy between the models' estimations within the ensemble that achieved state-of-the-art results.




\medskip \textbf{Lipschitz Continuity Methods.} Several works utilize the Lipschitz continuity to improve the robustness of discriminative models against adversarial examples~\citep{Andriushchenko, tsuzuku, yang2020}. \cite{barrett2021} apply the gradient-preserving transformations from~\cite{anil} in a similar to our approach manner. However, their main focus is to use Lipschitz mappings for certifiable robustness against adversarial examples.

\section{Methodology}

\subsection{Compactness of the Learned Latent Representation}

A usual assumption for VAE models is that the prior follows the standard normal distribution: $p(\mathbf{z}) =  \mathcal{N}(\mathbf{z}; \mathbf{0},\mathbf{I})$. It is a meaningful choice from the perspective of the generative process since it provides a clear and simple way of sampling. Moreover, it is a natural candidate for the ELBO objective's regularization term in learning a Gaussian posterior per each input of VAE. However, it additionally implies an infinite support of the latent prior. We show that such an assumption contradicts the UAT (\cref{thm:uat}).

\begin{lemma}
\label{lem:compactimage}
Let $f:\mathcal{X} \to \mathcal{Y}$ be a continuous mapping from a topological space $\mathcal{X}$ to a topological space $\mathcal{Y}$. If $\mathcal{X}$ is compact, then its image $f[\mathcal{X}]$ is also compact (for proof, see Appendix A).
\end{lemma}

\textbf{Hence, by combining both \cref{thm:uat} and \cref{lem:compactimage} it follows that the image of any DNN trained on a compact set is also compact. This conclusion contradicts the infinite support assumption of the standard normal prior in the case of VAEs. Any DNN used as an encoder will map all inputs to the compact subset of the latent space.}

In the case of in-distribution inputs, this conclusion may be considered subtle since all such inputs should be assigned the appropriate density under the model learned during the DNN training. However, it plays a significant role as soon as the model starts dealing with the OoD inputs. These are the different inputs that the model has not seen before and has not been able to generalize during training. Therefore, as it was demonstrated in~\cite{Lan2020PerfectDM} the model is not constrained in putting those inputs anywhere within the whole available support or, more precisely, within the learned image of the encoder mapping. The properties of the compactness of the latent space become of great importance. One of the essential questions concerns the locations where the model tends to map the OoD inputs within the image compact space. As it was demonstrated by~\citeauthor{nalisnick2018deep}, the DGMs and VAEs, in particular, tend to be overconfident with OoD inputs. There were several attempts to explain this type of behaviour~\cite{nalisnick2019detecting, kirichenko2020normalizing}, but none addressed the issue of compactness.

In this paper, we deliberately enforce the latent space's compactness. The reason for that is twofold. First, it should alleviate the contradiction above in the modeling assumption of the VAE by providing a principled way to set the compactness of the image of the learned mapping. Furthermore, the input support for the decoder also gets a compact space during training which is again in line with UAT. Second, it allows us to conduct experiments with the outliers' detection in the controlled environment with the desirable compactness properties so that all the holes will be located within the predefined boundaries.

In principle, this approach can be implemented utilizing the following two separate methods: $(i)$~by Alexandroff extension and $(ii)$~by setting a predefined Lipschitz constant of the encoder. The first method implies a change of the intrinsic curvature of the latent space by switching from a Euclidean to a non-Euclidean manifold. On the other hand, the second method allows keeping a flat latent space by only enforcing specific bounds on a mapped compact.

\subsection{Compactification of the Latent Space}

\subsubsection{Compactification of the Latent Space to the Hypersphere}

The Alexandroff extension (\cref{def:alexandroff}) of $\reals^n$ can be done by adjoining a single point at infinity, turning the flat Euclidean space into a hypersphere $\mathcal{S}^n$ embedded into $\reals^{n+1}$.

\begin{lemma}
\label{lem:sphereiscompact}
Let $\mathcal{S}^{n} := \{ \mathbf{x} \in \reals^{n+1} : ||\mathbf{x}|| = 1 \}$ be a hypersphere with radius $r=1$ centered at $\mathbf{0}$ and embedded in $\reals^{n+1}$ then $\mathcal{S}^{n}$ is compact (for proof, see Appendix B).
\end{lemma}

The appropriate type of distribution that can be utilized on the hyperspherical surface is the von Mises-Fisher distribution which is parameterized by mean $\mu$ and concentration $\kappa$. If the concentration parameter $\kappa$ is greater than zero, then the distribution has properties similar to normal; however, when $\kappa=0$, then it is a uniform distribution. It allows choosing the uniform prior and calculating the corresponding KL-divergence term for the regularization within the latent space. We utilize the same algorithm introduced by~\cite{hvae} for the sampling and reparametrization trick. They named it a Hyperspherical VAE (HVAE).

\subsubsection{Enforcing Compactness by Setting a Lipschitz Constant on the Encoder Mapping}

Although the Alexandroff extension of the Euclidean space to the \textbf{hypersphere} is theoretically appealing, it \textbf{has an issue with the surface area collapse}, which makes it infeasible to use \textbf{in high-dimensional settings} (see Appendix C). To alleviate these issues, we implement our own method of ensuring the compactness of the latent codes. This method is beneficial since it keeps the flat Euclidean space for the latent representation and provides the necessary means to control the boundaries of the resulting compact.

\begin{theorem}
\label{thm:inDistsSuppVsOutDistsSupp}
Image of an M-Lipschitz mapping $f:\mathcal{X} \to \mathcal{Z}$ from a compact $K \subseteq \mathcal{X}$ with $x,y \in K$: $\left\Vert f(x) - f(y)\right\Vert \leq M\left\Vert x-y \right\Vert$ is bounded by both a corresponding Lipschitz constant $M$ \textbf{and} by a radius $R$ of a closed ball in the input support.
\end{theorem}

\begin{proof} 
By the Heine-Borel (\cref{thm:heineborel}), a compact $K \subset \reals^n$ is closed and bounded, meaning that the set is contained in some closed ball with a finite radius $R$. Hence, for any $x,y \in K$:
\begin{equation}
\left\Vert x - y\right\Vert \leq R
\end{equation}

Therefore, by combining the two inequalities above, we get:
\begin{equation}
\left\Vert f(x) - f(y)\right\Vert \leq MR
\end{equation}

so the mapping $f$ is bounded by the constant $MR$.
\end{proof}

Note that it is necessary to consider three components simultaneously to set a bound on the DNN output: bounds of the input compact, a norm being used, and, finally, an M-Lipschitz constant. In this work, we normalize the input support to the following compact vector space: $[0,1]^n$. It conveniently allows constraining $R \leq 1$ by applying an $L_{\infty}$-norm.

Moreover, to preserve both the generative functionality and the comparable log-likelihood values with the non-compact latent prior, it is important to consider the properties of the standard normal prior distribution. In the case of the low-dimensional setting, it is natural to bind the resulting compact with some standard deviation multiplier depending on the condensing tightness one wants to obtain. However, in the high-dimensional setting, the typical set should be considered. For that reason, the actual values for the Lipschitz constant of the encoder should be based on the dimensionality of the latent space. Namely, an upper bound on the mapped image should depend on the location of the typical set of the prior and its width. Recall that the center of the typical set of a centered normal distribution is located at the distance of $\sigma\sqrt{m}$ from the mode. In our experiments, we set the width equal to two standard deviations, and we choose the closest whole number: 
\begin{equation}
M \coloneqq \lfloor\sigma\sqrt{m} + 2\sigma\rceil    
\end{equation}

where $m$ is the dimensionality of the latent representation and $\sigma = 1$ for the standard deviation of the prior.

In our work, we ensure the Lipschitz constant of the mapping utilizing the GroupSort activation function together with a projection of the weights of each layer on $L_{\infty}$-ball during the forward-pass of the DNN. The constant is set layerwise in the following way: for a DNN with $K$ number of layers in order to guarantee the $M$-lipschitz constant of the entire network mapping, we enforce the $M^{\frac{1}{K}}$ constant per each of its layer. It relies on the fact that the finite composition of Lipschitz functions is also Lipschitz with the product of the corresponding constants used in composition:
\begin{equation}
M= \prod_{n=1}^{K}M^{\frac{1}{K}}    
\end{equation}

The main building blocks are both 1-Lipschitz non-linearity and 1-Lipschitz linear mapping per each layer. The appropriate scaling of the results makes them equal to $M^{\frac{1}{K}}$. For the complete algorithm, see Appendix E.

\begin{figure*}[ht]
\centering 
\includegraphics[width=0.9\linewidth]{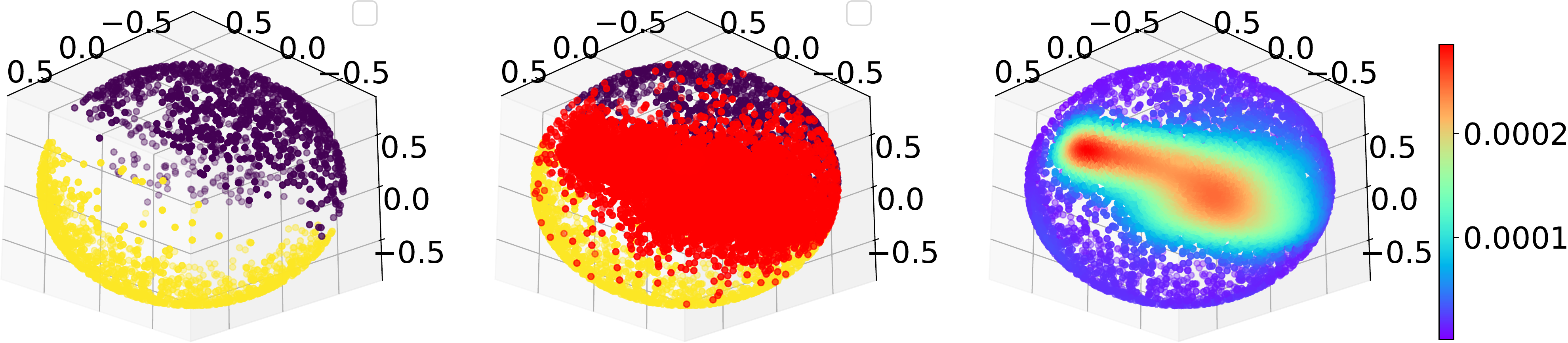}
\caption{Compact $S^2$ latent space of VAE trained only on the two digits of the MNIST dataset 0's and 1's. The outliers densely land on the hole.
\textbf{From left to right}: yellow depicts means of the estimated posteriors for 1's and purple for 0's; red represents the mapped means for a held-out outlier: a class of digits 9's; kernel density estimation of all means with the densest region in the hole packed with the outliers.}
\label{fig:HeldOutMNISTOnSphere}
\end{figure*}

\subsection{Latent holes}

We look at the holes from two different viewpoints as mentioned in section~\ref{latent_holes_duality}. First, we apply the following operational perspective to the definition of the hole: if two closely located latent points produce two distant samples in the input space, then we say that there is a hole in the latent space. This definition is similar to the one introduced by~\cite{falorsi2018}. Second, from the conceptual perspective, we treat the holes as the regions where there is a discrepancy between the aggregated posterior and the prior~\citep{xu2020}, i.e., the hole appears when the regions with the high prior density have a low density of the aggregated posterior. Despite the seemingly different motivations for both definitions, it has been demonstrated by~\cite{ruizhe2021} that they are, in fact, connected. Moreover, if it is possible \emph{to squeeze} all of the model's inputs within the high-density region of the prior, then the only ``free'' space within the latent compact turns out to be these holes.

\subsection{Why squeezing?}

The reason for that is at least two-fold. First, because of the arbitrariness of the mapping of outliers, it appears only logical to limit the whole image for any input (including outliers) within the same constrained space as for the inliers in order to eliminate this arbitrariness. The opposite approach, i.e., the widening of the compact, will not provide any benefits, only allowing for the model to use more ``free'' space where the outliers can be mapped to. Also, considering the well-known overconfidence issue~\cite{nalisnick2018deep}, the wide compact does not guarantee the usage by the model of this available free space for any input. Some of the inputs can indeed be placed in the available space far from the mode; however, some will still be placed close to the mode (see Figure~\ref{fig:lInfNormsNoLipschitzMNIST}). Second, recall that VAEs, beside being probabilistic models, are also autoencoders. So they can be viewed from the perspective of the information bottleneck principle, i.e., when the information is put under pressure using the low-dimensional bottleneck layer to extract the relevant factors of variations of the input data in question. The compactness can be considered as a supplementary constraint to the low latent dimensionality (note that the dimensionality is also a topological property). By low dimensionality, we mean in comparison with the dimensionality of the input. Hence, by putting additional pressure in the form of a tight condensing of the mapped image within the predefined limits of the compact, we force the model to learn the bounded factors of variations for \emph{any} input in a controlled and principled manner, eliminating the unnecessary ``free'' space for the model where it can potentially place outliers. The experimental evidence reveals that in such case the model indeed tends to place the outliers within the only available ``free'' space---the latent holes which, in turn, can be easily detected.

\subsection{Scores}

As we indicated before, currently available scores for the holes' detection are based either on the availability of the suitable metric in the input space~\citep{falorsi2018} or on the computationally expensive estimation of the aggregated posterior based on all the training samples~\cite{xu2020}. The motivation for that was clearly because these scores are based on the inlier inputs; hence the search for the holes starts from their corresponding latent codes. However, in the case of outlier detection, we can directly check if the mapped input lands within the hole. For this purpose, we sample the approximated posterior $q_{\boldsymbol{\phi}}(\rvz|\rvx)$ with several latent codes $\rvz$ under a particular input $\rvx$ and compute the sample standard deviation of the log-likelihoods $\log p(\rvx|\rvz)$ (see Appendix G). 

The approximated posterior under the input provides a locality within the latent space. Based on this locality---the samples from the posterior give us the notion of \emph{nearness} around this specific locality. Finally, the standard deviation of the log-likelihoods based on the samples indicates \emph{how far} from each other the sampled codes are mapped back into the input space. As a result, it becomes a beneficial indicator because it does not require making a particular traversal along some path (as was the case in~\citep{falorsi2018}) or doing a thorough search through the latent space for all available holes (as was done in~\citep{ruizhe2021}). On the contrary, it allows direct checking if we are within the hole or not for a particular input.

There is also an alternative but still connected way of scoring the presence of a hole. Recall that density calculation of the given input under probability models with latent variables can be done through marginal likelihood. It is defined as the expected model likelihood marginalized over the latent's prior:
\begin{equation}
p(\rvx) = \E_{p(\rvz)} [ p(\rvx|\rvz) ]
\end{equation}

First, let $\rvz \in \sL$ and $|\sL| < \infty$ then the marginal likelihood can be considered as a finite mixture of different $p(\rvx|\rvz)$ with different constant weights $w = p(\rvz)$ s.t. $\sum_{i=1}^{|\sL|}w_i = 1$: 
\begin{equation}
\label{eq:finite_marginal_likelihood}
p(\rvx) = \sum_{i=1}^{|\sL|} w_i p(\rvx|\rvz_i)
\end{equation}

And $|\sL|$ is the size of the components in the considered finite mixture of the likelihoods. Now suppose that $p(\rvx|\rvz)$ is fully factorized, then the variance of the mixture of individual random components $\ervx$'s comprising $\rvx$ is given by:

\begin{equation}
\begin{aligned}
\label{eq:variance_finite_mixture}
\operatorname{Var}_{p(\rvx)}(\ervx)=\underbrace{\sum_{i=1}^{|\sL|} w_{i} \Var_{p(\rvx|\rvz_{i})}[\ervx]}_{\text{Weighed individual variances}} + \\
+\underbrace{\sum_{i=1}^{|\sL|} w_{i}\left(\E_{p(\rvx|\rvz_{i})}[\ervx]\right)^{2}-\left(\sum_{i=1}^{|\sL|} w_{i} \E_{p(\rvx|\rvz_{i})}[\ervx]\right)^{2}}_{\text{Jensen's gap}}
\end{aligned}
\end{equation}

The first term is a weighted sum of variances of individual model likelihoods under all latent codes. Note that the difference of second and third terms is always non-negative due to Jensen's inequality. This difference represents a \emph{Jensens's gap} and can be interpreted as the variance of the means of the likelihoods weighted by the appropriate prior probabilities of the latent. Hence, by computing the variance of the marginal likelihood under importance sampling due to this Jensen's gap, it is possible to estimate the variance of the means of the likelihoods, which can be utilized for hole detection with outlier inputs. For that reason, we apply the sample standard deviation of the estimated marginal likelihoods under importance sampling (see Appendix G) and test the performance of this score in our thorough experiments. Since the marginal likelihood is already quite frequently estimated under importance sampling in many practical implementations, it becomes possible to quickly adapt these implementations for practitioners to incorporate the sample standard deviation of the marginal likelihood under importance sampling to get as a handy byproduct an alternative score for the hole identification. To distinguish between the two scores, we label the first as the hole indicator and the second as the standard deviation of marginal log-likelihoods (Stds of LLs for short).

\medskip \textbf{\textit{Threshold.}} For identifying the best threshold for the scores, we utilize threshold-independent metrics (these metrics are calculated for all possible thresholds) such as the \textit{Area Under the Receiver Operating Characteristic Curve} (AUROC),  the \textit{Area Under Precision-Recall curve} (AUPR), and the \textit{False-Positive Rate at 80\% of True-Positive Rate} (FPR80)~\citep{davis2006relationship}.

\begin{figure*}[h]
\centering
\begin{subfigure}[b]{0.5\columnwidth}
\includegraphics[width=\textwidth]{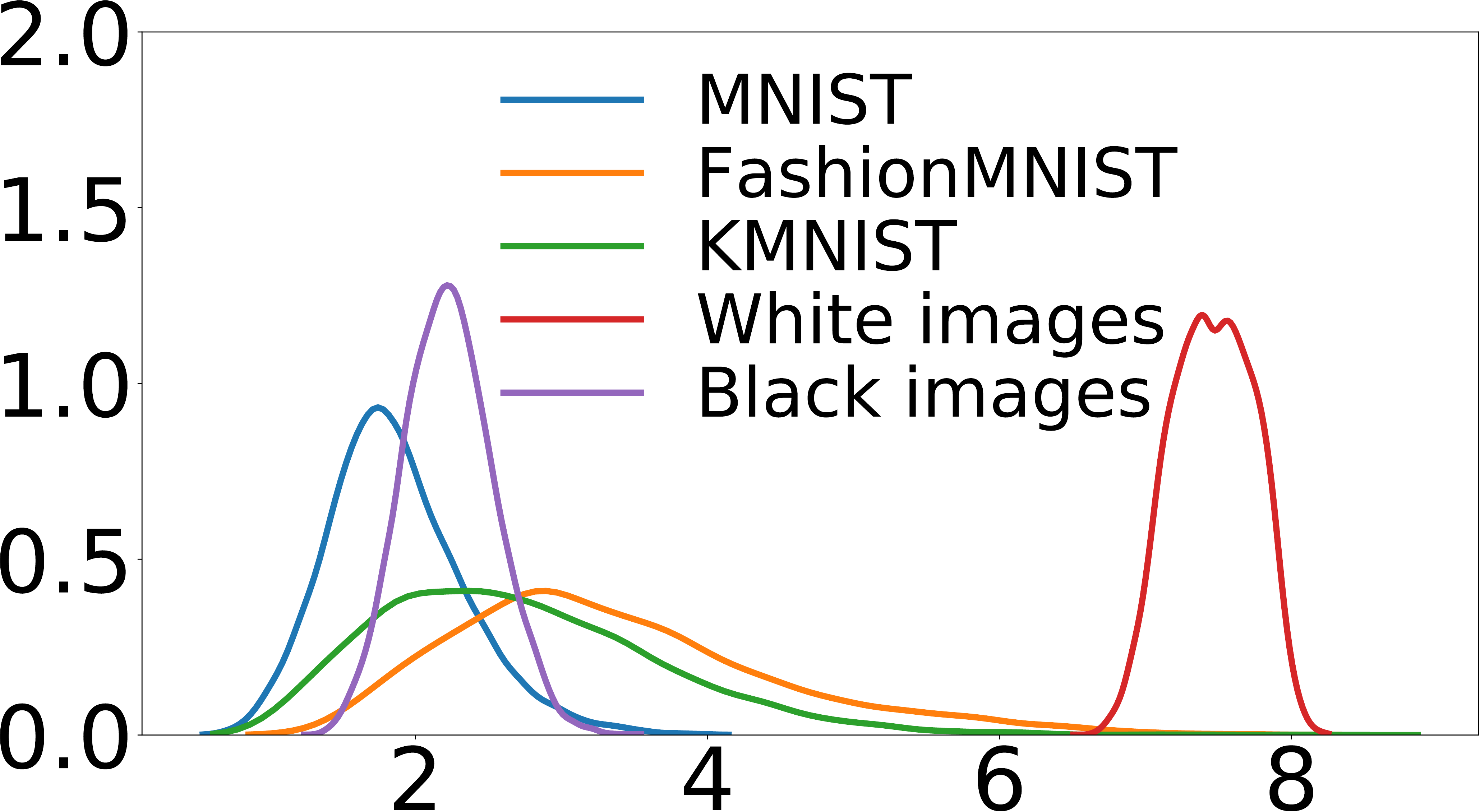}
\caption{Vanilla VAE (MNIST)}
\end{subfigure}
\begin{subfigure}[b]{0.471\columnwidth}
\includegraphics[width=\textwidth]{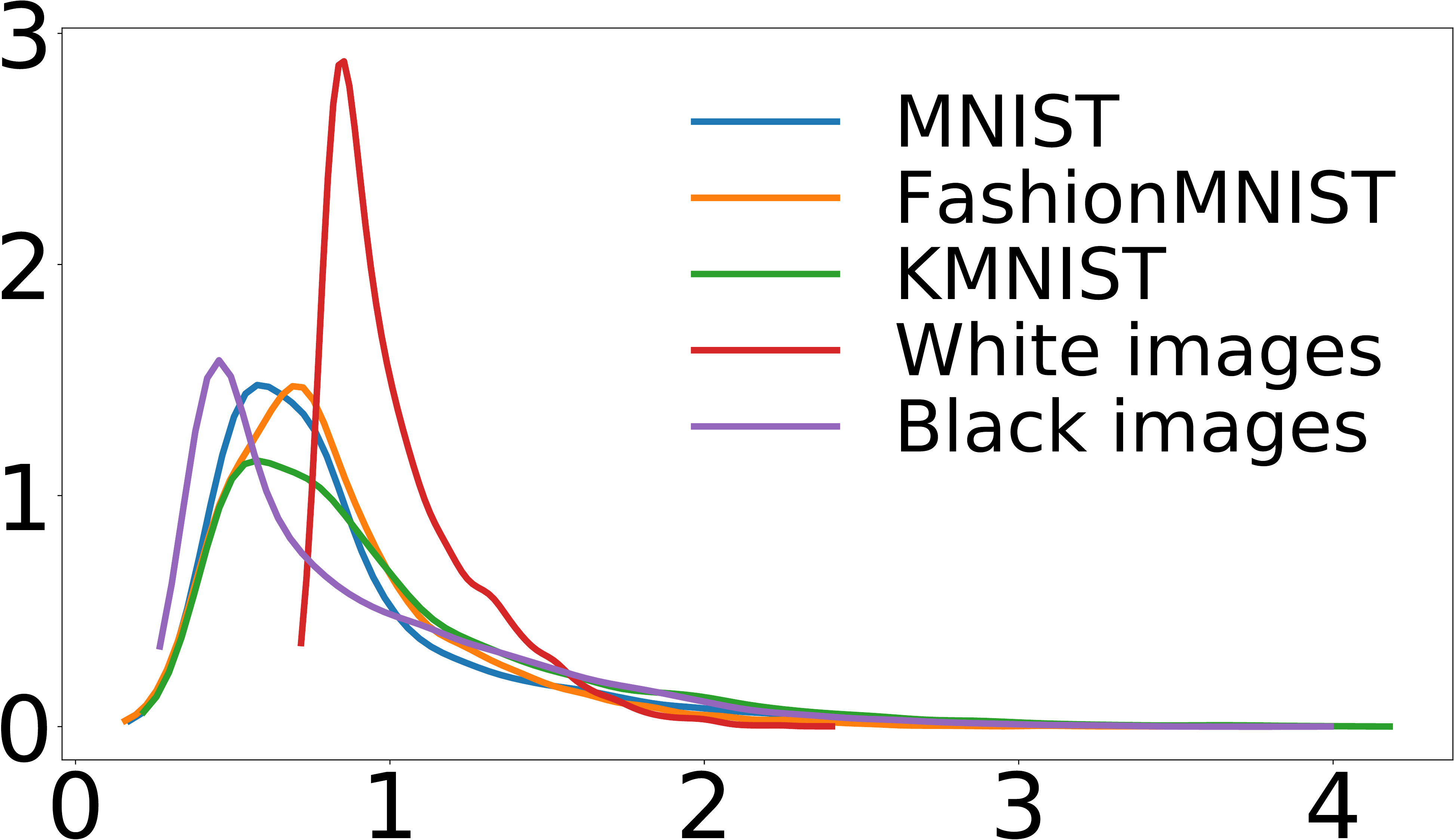}
\caption{Lipshitz VAE (MNIST)}
\end{subfigure}
\begin{subfigure}[b]{0.5\columnwidth}
\includegraphics[width=\textwidth]{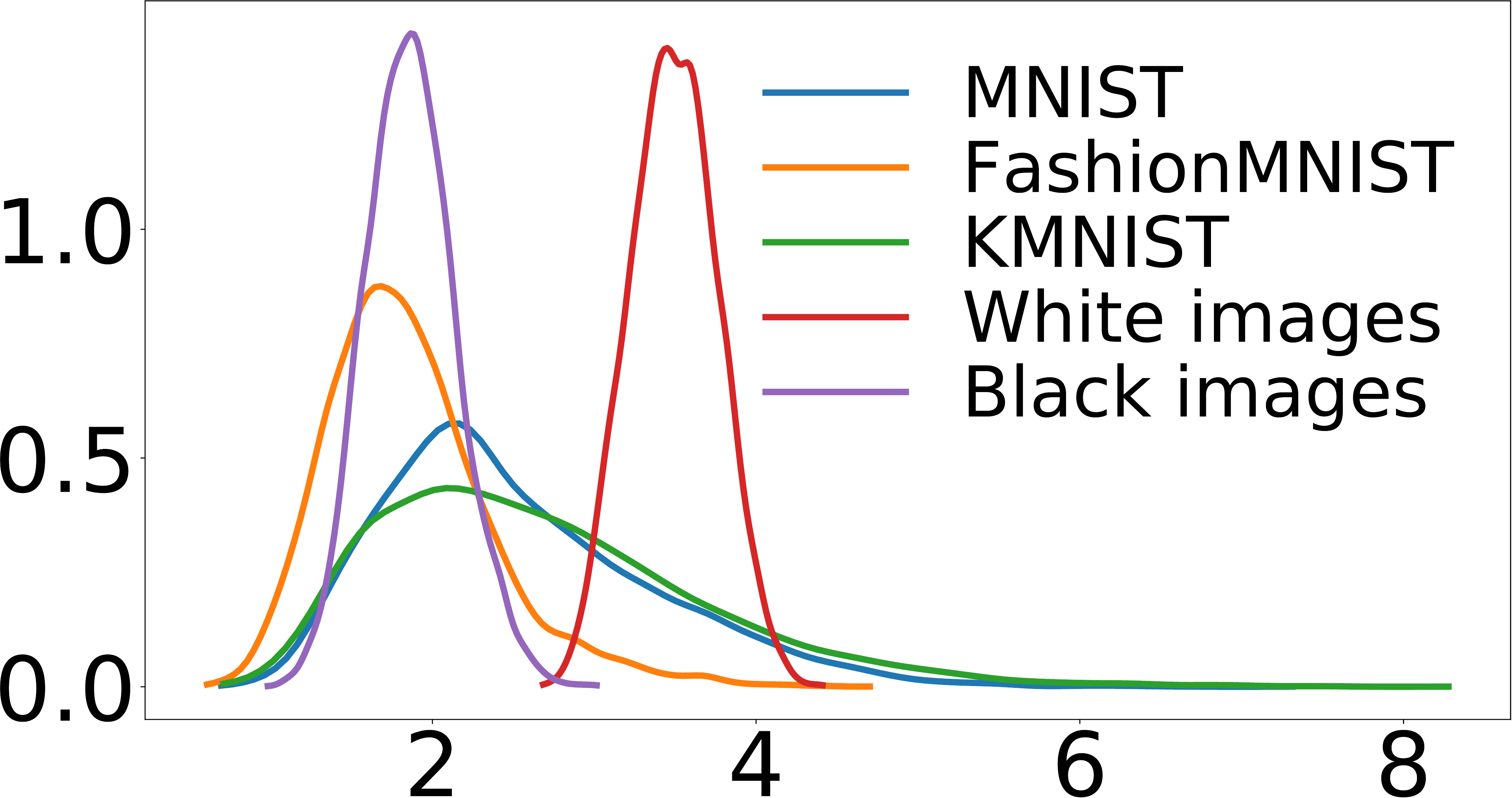}
\caption{Vanilla VAE (FMNIST)}
\end{subfigure}
\begin{subfigure}[b]{0.5\columnwidth}
\includegraphics[width=\textwidth]{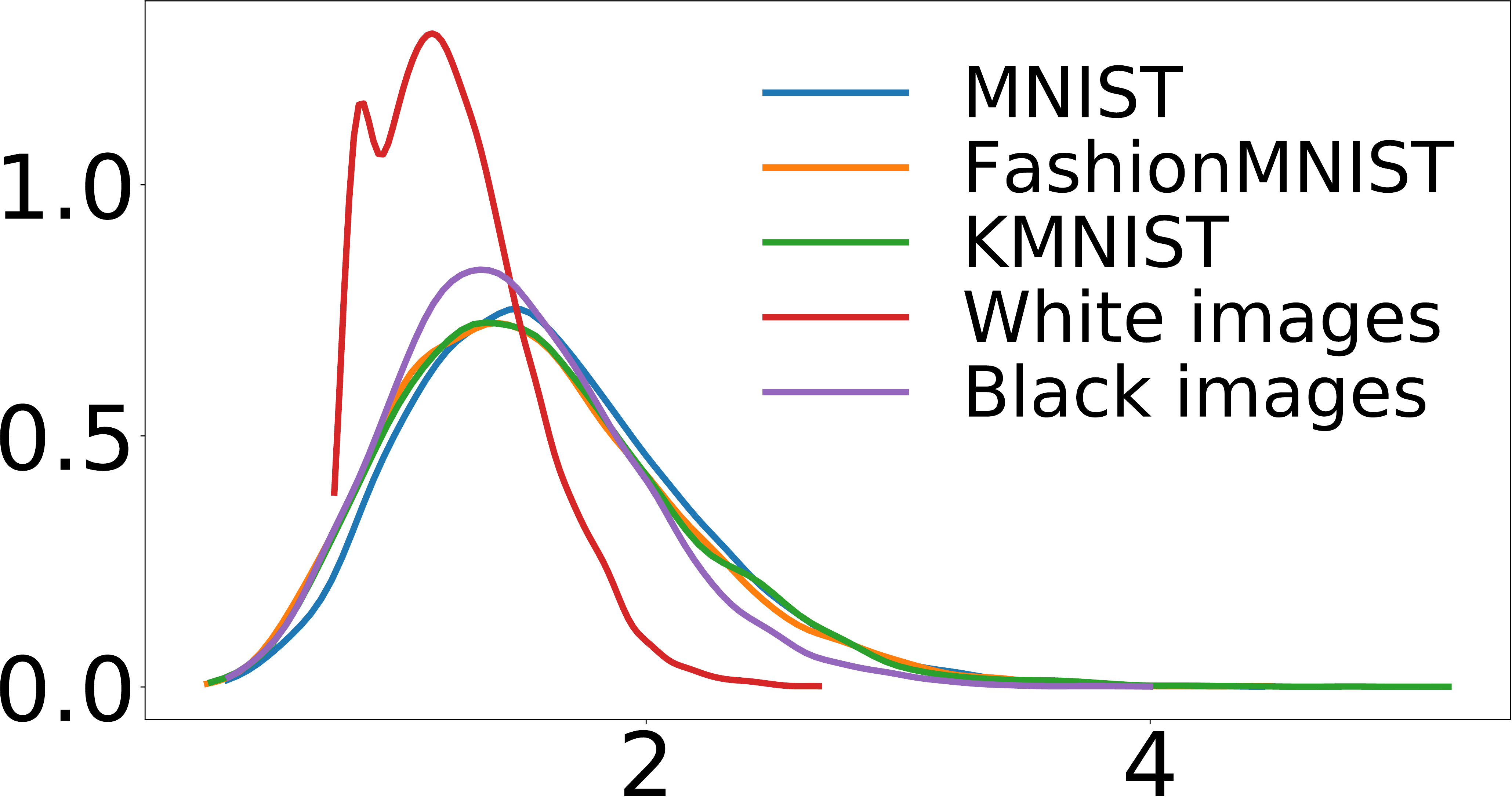 }
\caption{Lipschitz VAE (FMNIST)}
\end{subfigure}
\caption{Estimated Gaussian kernel densities of $L_\infty$-norms of the approximated posterior means in the latent space for datapoints from MNIST as inliers, and datapoints from FashionMNIST, KMNIST, white and black images as outliers. \textbf{From left to right}: classical VAE trained on MNIST; VAE with a fixed Lipschitz constant $M = 1$ for encoder trained on MNIST; classical VAE trained on Fashion-MNIST; VAE with a fixed Lipschitz constant $M = 1$ for encoder trained on Fashion-MNIST.}
\label{fig:lInfNormsNoLipschitzMNIST}
\end{figure*}

\section{Evaluation}

\subsection{Toy Experiments with Compact $\mathcal{S}^2$}

We begin with the simple held-out experiments based on the MNIST dataset~\citep{mnist}. For that reason, we utilize HVAE.~\footnote{The source code of the implemented solution is freely available at \url{https://github.com/DigitalDigger/VAEOutliersDetectionByVacantHoles}} It is trained with the hyperspherical uniform prior on $\mathcal{S}^2$ only on two digits as inliers, namely zeros and ones. The rest of the handwritten digits are considered outliers. These experiments assist in acquiring a fundamental intuition in the way how the encoder of the model maps the outliers in the compact latent space. As it can be observed in Figure~\ref{fig:HeldOutMNISTOnSphere}, the two inlier classes are separated from each other on the sphere surface. There is also a hole between the clusters formed by these classes. Next, we try to map to the latent space held-out classes. As a result, we visually demonstrate that the encoder is forced to place the unseen during training classes somewhere within the constrained space and choose to land the outliers into latent holes. It happens when the model is confronted with the bounded factors of variation. In addition, we run experiments with our hole detection score $\Sigma_{\rvz}[\rvx]$ first with all held-out classes as outliers and second with all classes of Fashion-MNIST~\citep{xiao2017} as outliers. In addition, we conduct the experiments for 10 separate runs and summarize the results in a $99.9\%$ confidence interval values that can be observed in Table~\ref{table:HypersphereMNIST}. The obtained result strongly support our hypothesis about holes as centers of attraction for the outliers. Moreover, we compare these results with the corresponding baseline scores using Vanilla VAEs with the same low dimensionality of the latent space and also benchmark the hole indicator on the model trained on all classes of Fashion-MNIST vs. all MNIST classes as outliers (see Appendix F).

\begin{table}[t]
\caption{Hole indicator (means and $99.9\%$ confidence interval values for 10 separate runs) for toy experiments with $\mathcal{S}^2$. The held-out outliers are all digits except 0's and 1's.}
\label{table:HypersphereMNIST}
\resizebox{\columnwidth}{!}{%
\begin{tabular}{lcc}
\toprule
& \textbf{MNIST held-out} & \textbf{MNIST vs. Fashion-MNIST} \\
\cmidrule(lr){2-2}\cmidrule(lr){3-3}
\textbf{ROC AUC$\uparrow$} &  89.05 (±0.25) &  94.54 (±0.09)\\
\textbf{AUPRC$\uparrow$} & 99.38 (±0.02) & 99.01 (±0.02) \\
\textbf{FPR80$\downarrow$} & 16.1 (±0.72) & 5.60 (±0.2)\\
\bottomrule
\end{tabular}
}
\end{table}

\begin{table*}[h]
\centering
\caption{Scoring values for the Lipschitz constrained VAEs trained on MNIST, Fashion-MNIST and CIFAR10}
\label{table:VanillaMNIST}
\resizebox{\textwidth}{!}{
\begin{tabular}{lccccccccc}
\toprule
& \multicolumn{3}{c}{\textbf{MNIST vs. Fashion-MNIST}} & \multicolumn{3}{c}{\textbf{Fashion-MNIST vs. MNIST}} & \multicolumn{3}{c}{\textbf{CIFAR10 vs. SVHN}} \\
\cmidrule(lr){2-4}\cmidrule(lr){5-7}\cmidrule(lr){8-10}
& \textbf{ROC AUC$\uparrow$} & \textbf{AUPRC$\uparrow$} & \textbf{FPR80$\downarrow$} & \textbf{ROC AUC$\uparrow$} & \textbf{AUPRC$\uparrow$} & \textbf{FPR80$\downarrow$}  &
\textbf{ROC AUC$\uparrow$} & \textbf{AUPRC$\uparrow$} & \textbf{FPR80$\downarrow$}
\\
\midrule
& \multicolumn{9}{c}{\textit{\textbf{Vanilla VAE}}} \\
\midrule
\textbf{Log likelihood} &  99.99 &  99.99 &  0.00 & 54.03 & 57.37 & 84.70 & 61.08 & 53.92 & 56.25 \\
\textbf{Input complexity} & 0.00 & 32.91 & 100.00 &  99.17 &  99.24 & \cellcolor[rgb]{0.9,0.9,0.9} 0.00 & \cellcolor[rgb]{0.9,0.9,0.9} 95.87 & \cellcolor[rgb]{0.9,0.9,0.9} 95.36 & \cellcolor[rgb]{0.9,0.9,0.9} 9.09 \\
\textbf{Typicality test} & \cellcolor[rgb]{0.9,0.9,0.9} 100.00 & \cellcolor[rgb]{0.9,0.9,0.9} 100.00 & \cellcolor[rgb]{0.9,0.9,0.9} 0.00 & 53.81 & 50.78 & 70.74 & 59.75 & 64.06 & 80.20 \\
\midrule
& \multicolumn{9}{c}{\textit{\textbf{Bayesian VAE}}} \\
\midrule
\textbf{WAIC} & 99.99 & 99.99 & 0.00 &  59.53 &  59.35 & 71.88 & 61.15 & 54.22 & 57.15 \\
\textbf{Disagreement score} & 98.95 & 99.01 & 0.23 &  96.44  &  97.22 & 1.11 & 81.16 & 84.82  & 38.47 \\
\textbf{Entropy} & 99.42 & 99.47 & 0.02 &  97.97 &  98.43 & 0.19 & 84.76 & 88.21 & 29.31 \\
\midrule
& \multicolumn{9}{c}{\textit{\textbf{Lipschitz VAE}}} \\
\midrule
\textbf{Stds of LLs} & 99.78 & 99.79 & 0.06 & 99.21 & 99.16 & 0.84 & 86.40  & 84.88 & 21.59\\
\textbf{Hole indicator (ours)} & \textbf{99.87} & \textbf{99.87} & \textbf{0.00} & \cellcolor[rgb]{0.9,0.9,0.9} \textbf{99.69} & \cellcolor[rgb]{0.9,0.9,0.9} \textbf{99.65} & \textbf{0.28} & \textbf{91.76} & \textbf{89.58} &  \textbf{12.30}\\
\bottomrule
\end{tabular}
}
{\scriptsize The most robust scores are in bold. The highest values are in gray.\\
\textbf{*} 0's in FPR80 are possible since it is a value for false-positive rate at 80\% of true-positive rate}
\end{table*}

\subsection{Exploring Compactness Enforced by Lipschitz Continuity}

We continue probing compactness properties based on the constrained Lipschitz mapping to the latent space. We run experiments with both the classical VAE models and the VAE models with the enforced Lipschitz constant $M=1$ for the encoder. We trained four separate models (for the used DNN architectures, see Appendix D): on MNIST and Fashion-MNIST, with and without continuity constraints---the dimensionality of the latent space across all models: $m=10$. We evaluate the means of the approximated posteriors for the outliers from KMNIST~\citep{tarin} (and analogously from Fashion-MNIST for the models trained on MNIST and vice versa). In addition, we run the same tests with the specially forged datasets. One contains non-realistic images, but all of their pixels tend to the black color; another contains images that tend to the white color. The idea behind the two latter datasets is that they represent extreme values of the compact support of the input data. As shown in Figure~\ref{fig:lInfNormsNoLipschitzMNIST}, the possible range of the values achievable by the classical VAE is considerably broad based on the limited number of the outlier datasets. For the model trained on MNIST, it goes as far as seven standard deviations from the mean.

Meanwhile, the unconstrained model trained on Fashion-MNIST has a range with a maximum of around four standard deviations. It demonstrates the arbitrariness of the mapped compact and its limits. Note, however, that when we bound the continuity of the encoder, then both inliers and outliers are squeezed together in a compact within the appropriate limits, which experimentally confirms our theoretical result (see~\cref{thm:inDistsSuppVsOutDistsSupp}). It allows the enforcement of a controlled and bounded compactness on the flat prior.


\subsection{Detecting Outliers}

As we indicated before, due to the surface collapse of the sphere, it is infeasible to use HVAE with high-dimensional priors. Hence, we apply the fixed Lipschitz mapping together with the appropriate input normalization (all inputs are normalized to $[0,1]^n$). We evaluate our approach against several baseline methods. For them, we choose the classical VAE, the ensemble-based VAE, namely, the one based on the Bayesian inference over the weights of the DNN, and several approaches based on the new scores, namely, typicality score and input complexity. For scoring the Bayesian VAE, we utilize three available scores: WAIC, a disagreement score, and entropy. Bayesian inference is implemented utilizing the Bayes by Backpropagation~\cite{blundell2015weight}. The corresponding hyper-parameters and the training protocol are based on the work by~\cite{glazunov2022do}. All models trained on MNIST and Fashion-MNIST have the dimensionality of the latent space equal to $10$, and models trained on the CIFAR10 dataset have the latent of $70$ dimensions. For our suggested Lipschitz-based model, we compute the appropriate Lipschitz constant for the decoder based on the dimensionality of the latent space in order to preserve the comparable log-likelihood values of the classical VAE and also to be able to sample the prior in a standard way. For MNIST and Fashion-MNIST, it is equal to $5$, and for CIFAR10, it is equal to $10$. The results can be observed in Table~\ref{table:VanillaMNIST}. Our hole indicator demonstrates the best results among the scores that consistently perform well across all datasets. Moreover, the standard deviation of the likelihoods is the second most robust score, which agrees with our theoretical derivation (see Equation~\ref{eq:variance_finite_mixture}). By robustness in this context, we mean the persistence of the state-of-the-art results, independent of the dataset used for training the model and testing for the outliers. For example, despite the high values for the typicality test on MNIST vs. Fashion-MNIST datasets and input complexity on CIFAR10 vs. SVHN datasets, they are inconsistent across all of the considered datasets, making them unreliable in practical applications. The reason behind our score's robustness is that the model maps the outliers to the holes in the compact latent space (i.e., the only ``free'' space available for the learned factors of variations) that can be easily detected. Other scores rely either on the complexity of the dataset (as input complexity score), which is a data-dependent score, or on the hypothesis about the typical set, which is not always guaranteed because of the arbitrariness of the mapping of the encoder to any available ``free'' space including the holes in the typical set.

\subsection{Ablation Study}

To check if the continuity holes are responsible for the obtained results, we conduct experiments with the gradual reduction of the holes in the latent space. This can be done by smoothing out the decoder mapping. This approach is advantageous since it affects all holes in the latent space. Hence, if our assumption is correct, then the results of the outlier detection based on the holes should degrade according to the strength of the smoothing. We enforce smoothing by setting the corresponding Lipschitz constants on the decoder mapping in the same way as it was done for the encoder in previous experiments. We train six separate models, all of which have the Lipschitz encoder with $M=1$. Decoder is enforced with the values of the Lipschitz constants $M$ from the following set: $\{1, 2, 3, 4, 5, 10\}$. As can be seen in Figure~\ref{fig:decodeSmoothing}, there is an apparent performance degradation of the hole indicator for the outliers with decreasing of the corresponding Lipschitz constant enforced on the decoder, which is in line with our hypothesis that outliers land on the latent holes. Finally, we separately ablated the compactness component (for the results see Appendix H).


\begin{figure}[!h]
\centering
\includegraphics[width=1\columnwidth]{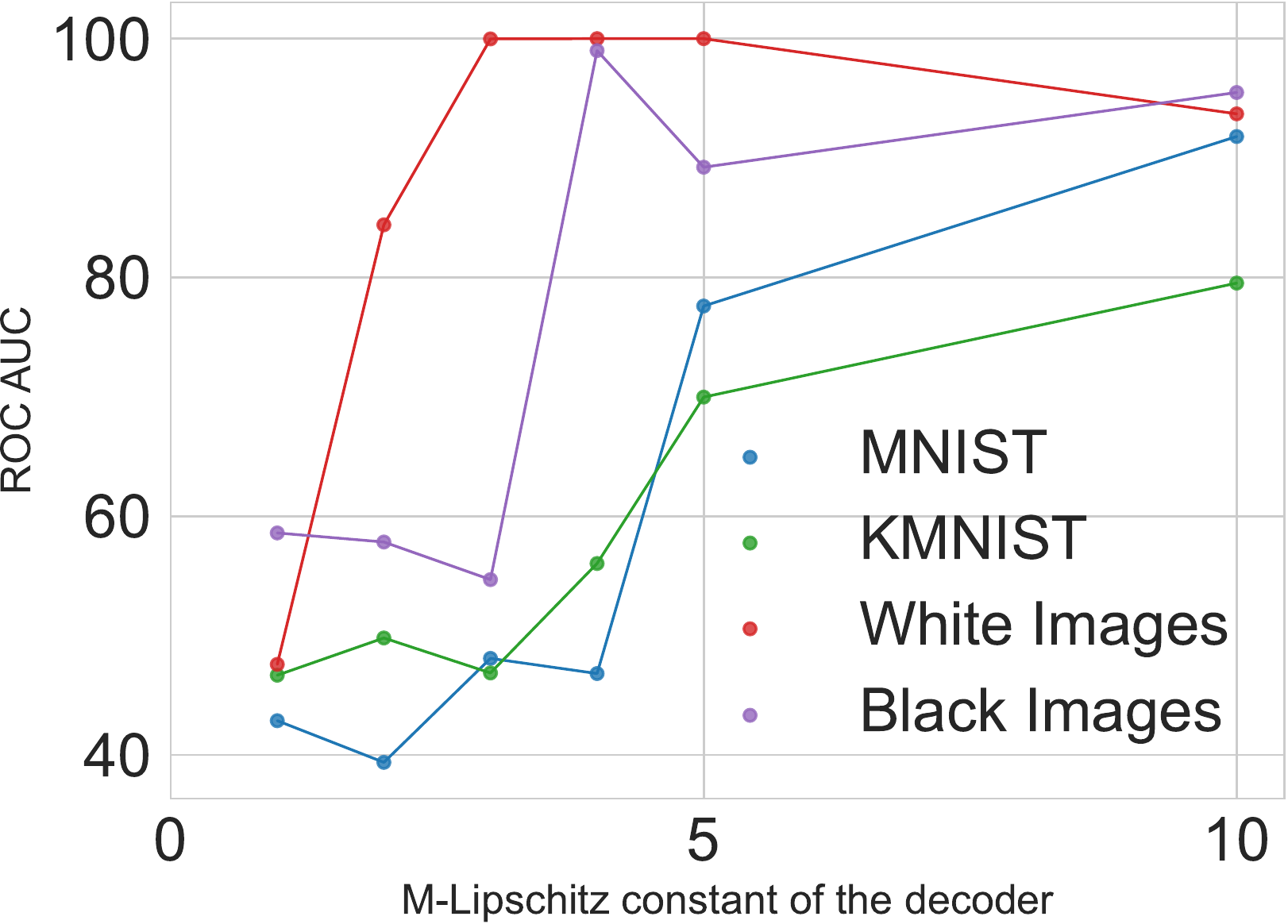}
\caption{ROC AUC values from the ablation study. VAE with different Lipschitz constants enforced on the decoder, namely, $M=1, M=2, M=3, M=4, M=5 \ \textrm{and} \ M=10$, all plotted along $x$-axis. VAE is trained on Fashion-MNIST with the Encoder Lipschitz constant $M = 1$ for all tests and evaluated on several outlier datasets.}
\label{fig:decodeSmoothing}
\end{figure}

\section{Conclusion}

In this paper, we identified an implicit theoretical inconsistency from the perspective of general topology between the VAE modeling and the UAT. We addressed this discrepancy utilizing the compactness of the mapped image to the latent space. In order to enforce the compactness, we devised a provable method for controlling the bounds of the resulting compact. The experimental evidence revealed that constraining the limits of the factors of variation is beneficial for outlier detection. In particular, we discovered that outlier inputs tend to be mapped to the latent continuity holes. By devising a special score for the hole indicator, we conducted several experiments aimed at their detection. Utilizing this score, we achieved promising results in unsupervised outlier detection based on the latent representation. Specifically, the suggested method and score demonstrated the most robust performance across all the used benchmarks and datasets.

\begin{acknowledgements}
This project has been partially funded from the European Union's research and innovation programmes under grant agreements No. 883275 (HEIR) and No. 101092912 (\hbox{MLSysOps}).
\end{acknowledgements}

\bibliography{paper}

\appendix

\section{Preservation of Compactness Under Continuous Mapping}

$\textbf{Lemma:}$ Let $f:\mathcal{X} \to \mathcal{Y}$ be a continuous mapping from a topological space $\mathcal{X}$ to a topological space $\mathcal{Y}$. If $\mathcal{X}$ is compact then its image $f[\mathcal{X}]$ is also compact.

$\textbf{Proof:}$\footnote{Adapted from: \url{http://mathonline.wikidot.com/preservation-of-compactness-under-continuous-maps}} Let $\mathcal C = \{ U_i \}_{i \in I}$ be any open covering of $f[\mathcal{X}]$ in $\mathcal{Y}$. Then: $f[\mathcal{X}] \subseteq \bigcup_{i \in I} U_i $


Now let us take the inverse of both its sizes:

\begin{align} 
\mathcal{X} \subseteq f^{-1} \left ( \bigcup_{i \in I} U_i \right ) \\  \mathcal{X} \subseteq\bigcup_{i \in I} f^{-1}(U_i)
\end{align}
Since $f$ is continuous and $U_i$ is open in $\mathcal{Y}$ for all $i \in I$ we have that $f^{-1}(U_i)$ is open in $\mathcal{X}$ for all $i \in I$. From above, we see that then $\{ f^{-1}(U_i) \}_{i \in I}$ is an open cover of $\mathcal{X}$. Since $\mathcal{X}$ is compact, this open cover has a finite subcover, say $\{ f^{-1}(U_{i_1}), f^{-1}(U_{i_2}), ..., f^{-1}(U_{i_n}) \}$ where $i_n \in I$ where:
\begin{align} \quad \mathcal{X} \subseteq \bigcup_{k=1}^{n} f^{-1}(U_{i_k}) \end{align}

Taking the image of both sides above and we have that:
\begin{align} \quad f[\mathcal{X}] \subseteq f \left ( \bigcup_{k=1}^{n} f^{-1}(U_{i_k}) \right ) \\ \quad f[\mathcal{X}] \subseteq\bigcup_{k=1}^{n} f(f^{-1}(U_{i_k})) \\ \quad f[\mathcal{X}] \subseteq \bigcup_{k=1}^{n} U_{i_k}
\end{align}

Thus $\mathcal C^* = \{ U_{i_1}, U_{i_2}, ..., U_{i_n} \}$ is a finite subcover of $\mathcal C$. Hence $f[\mathcal{X}]$ is compact in $\mathcal{Y} \hfill \qedsymbol$

\begin{figure}[ht]
  \centering
  \includegraphics[width=\columnwidth]{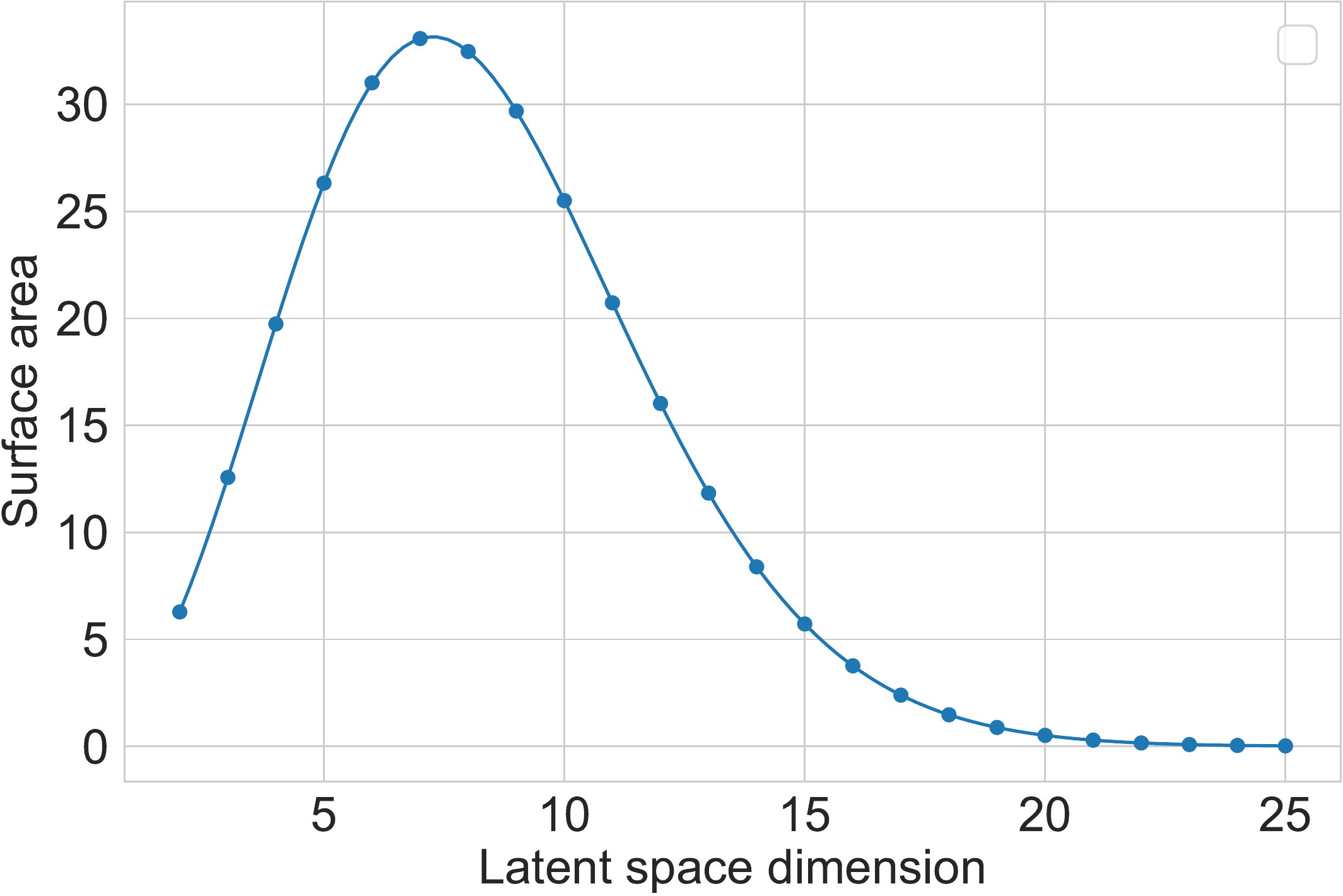 }
  \caption{The problem of surface area collapse.}
  \label{fig:surfaceAreaCollapse}
\end{figure}
\section{Sphere is compact}
$\textbf{Lemma:}$ Let $\mathcal{S}^{n} := \{ \mathbf{x} \in \reals^{n+1} : ||\mathbf{x}|| = 1 \}$ be a hypersphere with radius $r=1$ centered at $\mathbf{0}$ and embedded in $\reals^{n+1}$ then $\mathcal{S}^{n}$ is compact

$\textbf{Proof:}$ First note that $\mathcal{S}^n$ is obviously bounded. Next, observe that $||\rvx|| = \sum \ervx^2$ which represents a continuous mapping whose inverse is a closed set: $\{ 1 \}$, therefore the $\mathcal{S}^n$ is closed. It follows that the $\mathcal{S}^n$ is both closed and bounded, hence by Heine-Borel theorem it is compact.

\section{Surface Area Collapse of the Sphere}

As can be observed from the Figure~\ref{fig:surfaceAreaCollapse} the surface area grows up to approximately seven dimensions and after that it goes down completely collapsing in cases with greater than twenty dimensions. This issue makes it infeasible to use compact hyperspherical latent space in high-dimensional configurations.

\section{DNN Architectures Used}

For MNIST and FashionMNIST datasets with a single channel we used the following architectures for baseline experiments.

\begin{table}[!h]
\centering
\caption{Encoder CNN for MNIST and FashionMNIST}
\label{table:FMNISTEncoderCNN}
\begin{tabular}{lccc}
\toprule
\textbf{Operation} &  \textbf{Kernel} & \textbf{Strides} & \textbf{Feature Maps}  \\
\midrule
Convolution &  3 x 3 & 1 x 1 &  32\\
Convolution & 3 x 3 & 1 x 1 & 16 \\
Max pooling 2D & 2 x 2 & 2 x 2 & \textemdash \\
Linear for $\boldsymbol{\mu}$ & \textemdash & \textemdash & 10 \\
Linear for $\log \boldsymbol{\sigma}$ & \textemdash & \textemdash & 10 \\
\bottomrule
\end{tabular}
\end{table}

\begin{table}[!h]
\centering
\caption{Decoder CNN for MNIST and FashionMNIST}
\label{table:FMNISTDecoderCNN}
\resizebox{\columnwidth}{!}{
\begin{tabular}{lccc}
\toprule
\textbf{Operation} &  \textbf{Kernel} & \textbf{Strides} &\textbf{ Feature Maps}  \\
\midrule
Linear for sampled $\mathbf{z}$ & \textemdash & \textemdash & 2306 \\
Upsampling nearest 2D & \textemdash & \textemdash & \textemdash \\
Max pooling 2D & 2 x 2 & 2 x 2 & \textemdash \\
Transposed Convolution &  3 x 3 & 1 x 1 &  32\\
Transposed Convolution & 3 x 3 & 1 x 1 & 1\\
\bottomrule
\end{tabular}
}
\end{table}

For CIFAR10 dataset with three channels we used the following architectures with additional padding = 1 and no bias for every convolutional layer. Latent dimensionality = 70.

\begin{table}[!h]
\centering
\caption{Encoder CNN for SVHN and CIFAR10}
\label{table:SVHNEncoderCNN}
\begin{tabular}{lccc}
\toprule
\textbf{Operation} &  \textbf{Kernel} & \textbf{Strides} & \textbf{Feature Maps}  \\
\midrule
Convolution &  3 x 3 & 1 x 1 &  16\\
Convolution & 3 x 3 & 2 x 2 & 32  \\
Convolution & 3 x 3 & 1 x 1 & 32  \\
Convolution & 3 x 3 & 2 x 2 & 16  \\
Linear & \textemdash & \textemdash & 512 \\
Linear for $\boldsymbol{\mu}$ & \textemdash & \textemdash & 70 \\
Linear for $\log \boldsymbol{\sigma}$ & \textemdash & \textemdash & 70  \\
\bottomrule
\end{tabular}
\end{table}

\begin{table}[!h]
\centering
\caption{Decoder CNN for SVNH and CIFAR10}
\label{table:SVHNDecoderCNN}
\resizebox{\columnwidth}{!}{
\begin{tabular}{lccc}
\toprule
\textbf{Operation} &  \textbf{Kernel} & \textbf{Strides} & \textbf{Feature Maps}  \\
\midrule
Linear for sampled $\mathbf{z}$ &  \textemdash & \textemdash &  512\\
Linear & \textemdash & \textemdash &  1024\\
Transposed Convolution & 3 x 3 & 2 x 2 & 32  \\
Transposed Convolution & 3 x 3 & 1 x 1 & 32  \\
Transposed Convolution & 3 x 3 & 2 x 2 & 16  \\
Transposed Convolution & 3 x 3 & 1 x 1 & 3  \\
\bottomrule
\end{tabular}
}
\end{table}

For all architectures we used ReLU as a non-linearity in case of classical VAE. For Lipschitz encoder we used GroupSort. In addition, all pixels of the images have been normalized to [0,1] range for each channel for both training and testing phases. For HVAE we used the same architectures as in the original implementation~\footnote{We used the official implementation available at \url{https://github.com/nicola-decao/s-vae-pytorch}}, i.e., two hidden linear layers for the encoder with the dimensionality 256 and 128 correspondingly, and two hidden linear layers for the decoder with dimensionality 128 and 256. For Lipschitz VAE we also used two hidden linear layers for both encoder and decoder with doubled dimensionality for each corresponding hidden layer.

\section{Forward Pass of the Lipschitz Constant Enforcing}

\begin{algorithm}[!h]
\SetAlgoLined
\DontPrintSemicolon
\SetKwInOut{Input}{Input}
\SetKwInOut{Output}{Output}
\SetKwInOut{Result}{Result}
\SetKwInOut{Requires}{Requires}
\SetKwFunction{LInfBallProj}{LInfBallProjection}
\SetKwBlock{Forward}{Forward~pass}{Forward~pass}
\SetKwBlock{ProjectBlock}{LInfBallProjection}{LInfBallProjection}

\Input{Data point $\mathbf{x}$}
\Result{Network output $\mathbf{h}_L$}
\Requires{Lipschitz constant $M$}

\BlankLine
\SetKwFunction{GroupSort}{GroupSort}

\ProjectBlock{
    \Input{$\mathbf{y} \in \mathbb{R}^N$}
    \Output{$\mathbf{x} \in \mathbb{R}^N$}

    Sort $\mathbf{y}$ into $\mathbf{u}$: $u_1 \geq \ldots \geq u_N $\;
    
    Set $K := \max_{1 \leq k \leq N}\{k | (\sum_{r=1}^{k} u_r - 1) / k < u_k \}$\;
    
    Set $\tau := (\sum_{k}^{K}u_k - 1)/K$\;
    
    \For{$n=1$ \KwTo $N$}{
        Set $x_n := \max(y_n - \tau, 0)$\;
    }
}

\BlankLine
\Forward{
    $\mathbf{h}_0 \leftarrow \mathbf{x}$\;
    
    \For{$l=1$ \KwTo $L$}{
        $\mathbf{W}_l \leftarrow \LInfBallProj(\mathbf{W}_l)$\;
        
        pre-activation $\leftarrow M^{\frac{1}{L}} \mathbf{W}_l\mathbf{h}_{l-1}$\;
        
        $\mathbf{h}_l \leftarrow \GroupSort(\text{pre-activation})$\;
    }
}

\caption{Ensuring Lipschitz constant in a DNN mapping}
\label{alg: training_forward_pass}
\end{algorithm}

\section{Further Experiments with Hyperspherical VAE}

\begin{table*}[h]
\centering
\caption{Scoring values (means and $99.9\%$ confidence interval) for toy experiments with $\mathcal{S}^2$ for MNIST vs. held-out and Fashion-MNIST. The held-out outliers are all digits except 0's and 1's. And with $\mathcal{S}^3$ for Fashion-MNIST vs. MNIST. Note that Vanilla VAEs in the experiments are equipped with the same low dimensional latent space as the surface of the corresponding $\mathcal{S}$-VAE.} 
\label{table:ToyExperiments}
\resizebox{\textwidth}{!}{
\begin{tabular}{lccccccccc}
\toprule
& \multicolumn{3}{c}{\textbf{MNIST held-out}} & \multicolumn{3}{c}{\textbf{MNIST vs. Fashion-MNIST}} & \multicolumn{3}{c}{\textbf{Fashion-MNIST vs. MNIST}} \\
\cmidrule(lr){2-4}\cmidrule(lr){5-7}\cmidrule(lr){8-10}
& \textbf{ROC AUC$\uparrow$} & \textbf{AUPRC$\uparrow$} & \textbf{FPR80$\downarrow$} & \textbf{ROC AUC$\uparrow$} & \textbf{AUPRC$\uparrow$} & \textbf{FPR80$\downarrow$}  &
\textbf{ROC AUC$\uparrow$} & \textbf{AUPRC$\uparrow$} & \textbf{FPR80$\downarrow$}
\\
\midrule
& \multicolumn{9}{c}{\textit{\textbf{Vanilla VAE}}} \\
\midrule
\textbf{Log likelihood} &  96.84 (±0.07) &  98.50 (±0.04) &  4.43 (±0.27) & 99.85 (±0.02) & 99.86 (±0.01) & 0.00 (±0) & 45.13 (±0.1) & 43.75 (±0.05) & 75.60 (±0.27) \\
\textbf{Input complexity} & 42.98 (±0.86) & 45.28 (±0.52) & 81.82 (±0) &  18.27 (±2.12) &  37.18 (±0.8) & 100 (±0) & \cellcolor[rgb]{0.9,0.9,0.9} 94.96 (±1.18) & \cellcolor[rgb]{0.9,0.9,0.9} 95.57 (±1.12) & \cellcolor[rgb]{0.9,0.9,0.9} 10.91 (±5.68) \\
\textbf{Typicality test} & \cellcolor[rgb]{0.9,0.9,0.9} 96.84 (±0.05) & 98.50 (±0.04) & \cellcolor[rgb]{0.9,0.9,0.9} 4.24 (±0.25) & \cellcolor[rgb]{0.9,0.9,0.9} 99.86 (±0.01) & \cellcolor[rgb]{0.9,0.9,0.9} 99.87 (±0.01) & \cellcolor[rgb]{0.9,0.9,0.9} 0.00 (±0) & 45.16 (±0.1) & 43.76 (±0.06) & 75.60 (±0.35) \\
\midrule
& \multicolumn{9}{c}{\textit{\textbf{$\mathcal{S}$-VAE}}} \\
\midrule
\textbf{Log likelihood} &  97.07 (±0.05) &  98.62 (±0.06) &  4.34 (±0.24) & 99.85 (±0.02) & 99.87 (±0.01)  & 0.01 (±0.01) & 45.25 (±0.07) & 44.45 (±0.05) & 76.21 (±0.26) \\
\textbf{Input complexity} & 41.74 (±1.11) & 44.67 (±0.44) & 80.00 (±5.68) &  17.54 (±2.45) &  37.02 (±0.83) & 100 (±0) & \cellcolor[rgb]{0.9,0.9,0.9} 94.79 (±1.63) & \cellcolor[rgb]{0.9,0.9,0.9} 95.45 (±1.39) & \cellcolor[rgb]{0.9,0.9,0.9} 12.73 (±7.57) \\
\textbf{Typicality test} & \cellcolor[rgb]{0.9,0.9,0.9} 97.04 (±0.05) & 98.59 (±0.05) & \cellcolor[rgb]{0.9,0.9,0.9} 4.34 (±0.25) & \cellcolor[rgb]{0.9,0.9,0.9} 99.86 (±0.02) & \cellcolor[rgb]{0.9,0.9,0.9} 99.87 (±0.02)  & \cellcolor[rgb]{0.9,0.9,0.9} 0.00 (±0) & 45.25 (±0.08) & 44.45 (±0.09) & 76.17 (±0.24) \\
\textbf{Hole indicator (ours)} & \textbf{89.05 (±0.25)} & \cellcolor[rgb]{0.9,0.9,0.9} \textbf{99.38 (±0.02)} & \textbf{16.1 (±0.72)} &  \textbf{94.54 (±0.09)} & \textbf{99.01 (±0.02)} & \textbf{5.60 (±0.2)} & \textbf{87.37 (±0.16)} & \textbf{88.86 (±0.15)} &  \textbf{19.25 (±0.46)}\\
\bottomrule
\end{tabular}
}
{\scriptsize The most robust scores are in bold. The highest values are in gray.\\
\textbf{*} 0's in FPR80 are possible since it is a value for false-positive rate at 80\% of true-positive rate}
\end{table*}

As can be observed in Table~\ref{table:ToyExperiments} the most robust scores are hole indicators that achieve the most consistent results across all used datasets.

\section{Scores}

\subsection{Stds of LLs}

Recall that \emph{importance sampling} is used to estimate the marginal likelihood of the input under the trained VAE, namely:
\begin{footnotesize}
\begin{equation}\label{eq_importance_sampling}
p_{\boldsymbol{\theta}}(\mathbf{x}) \simeq \frac{1}{N} \sum_{i=1}^{N} \frac{p_{\boldsymbol{\theta}}(\rvx, \rvz_{\textmd{ \emph{(i)}}})}{q_{\boldsymbol{\phi}}(\rvz_{\textmd{ \emph{(i)}}} | \rvx)},\quad\textrm{where}\quad\rvz_{\textmd{ \emph{(i)}}} \sim q_{\boldsymbol{\phi}}(\rvz | \rvx)
\end{equation}
\end{footnotesize}

where $\boldsymbol{\phi}$ represents the variational parameters of the encoder responsible for the variational approximation of the posterior $q_{\boldsymbol{\phi}}$ over the latent variable $\mathbf{z}$, and $\boldsymbol{\theta}$ stands fr the generative parameters of the decoder responsible for the parametrization of the likelihood of the input $p_{\boldsymbol{\theta}}(\mathbf{x}|\mathbf{z})$.
Hence, it is possible to compute the sample standard deviation of the marginal likelihood under \emph{importance sampling} by computing the sample standard deviation of the terms within the given sum. This constitutes the essence of the Stds of LLs score.

\subsection{Hole Indicator Score}

For this score we sample the approximated posterior $q_{\boldsymbol{\phi}}(\rvz|\rvx)$ with several latent codes $\rvz$ under a particular input $\rvx$ and compute the sample standard deviation of the log-likelihoods $\log p(\rvx|\rvz)$:
\begin{equation}
\Sigma_{\rvz}[\rvx] = \sqrt{\frac{1}{N-1} \sum_{{\rvz}} \left(\log p(\rvx|\rvz) - \overline{\log p(\rvx|\rvz)}\right)^2}\end{equation}

\subsection{Typicality}

The test for typicality treats all input sequences as inliers if their entropy is sufficiently close to the entropy of the model, i.e., if the following holds for small $\epsilon$ then the given input is in-distribution: 
\begin{footnotesize}
\begin{equation}
\begin{split}
\left|-\log p\left({\rvx^*}\right)-\sum_{\rvx \in \mathcal{D}}\log p(\rvx)\right| \leq \epsilon
\end{split}
\end{equation}
\end{footnotesize}
This score is applied to one-element sequences in our work since it is the most realistic scenario in practical applications of outlier detection.

\subsection{Input Complexity}

\begin{figure*}[ht]
\centering
\begin{subfigure}[b]{0.32\textwidth}
\includegraphics[width=\textwidth]{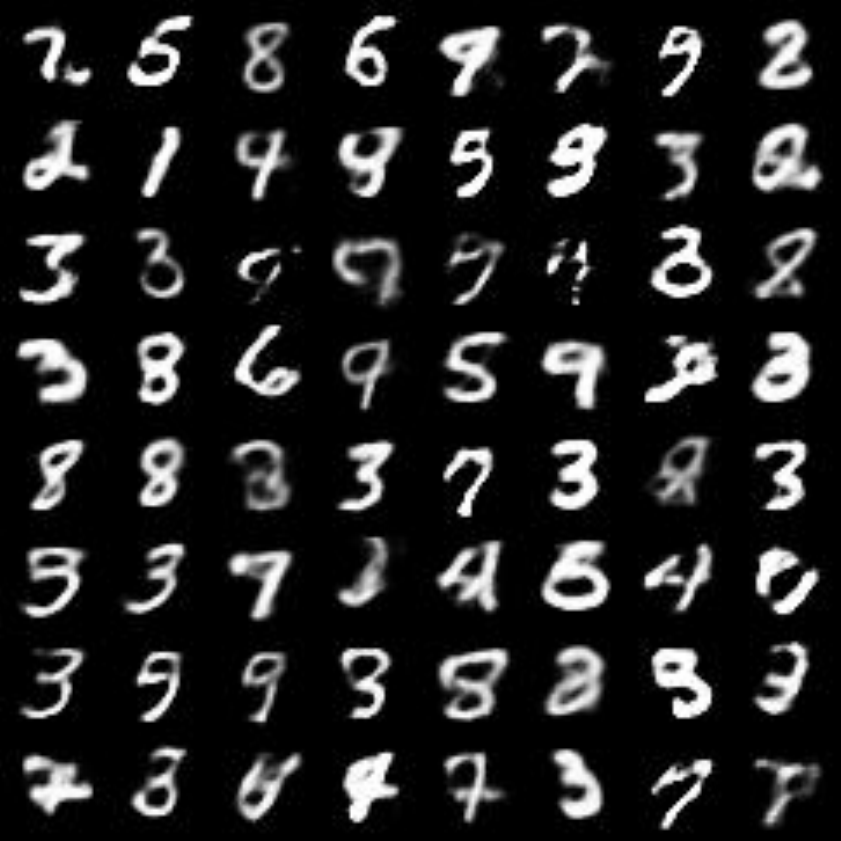}
\end{subfigure}
\hfill
\begin{subfigure}[b]{0.32\textwidth}
\includegraphics[width=\textwidth]{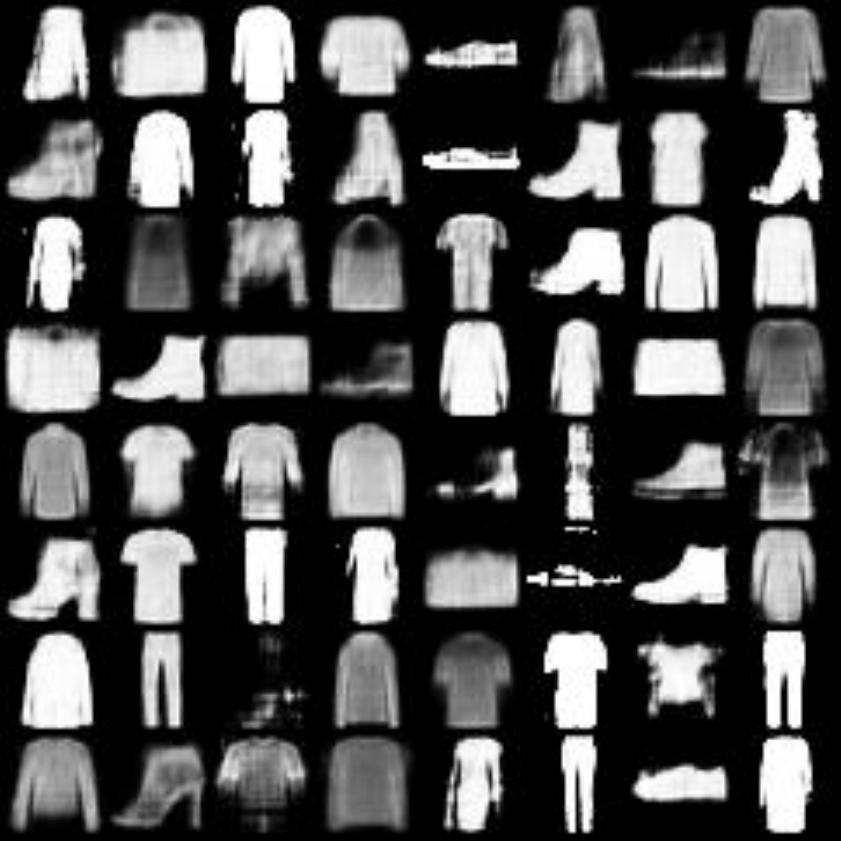}
\end{subfigure}
\hfill
\begin{subfigure}[b]{0.32\textwidth}
\includegraphics[width=\textwidth]{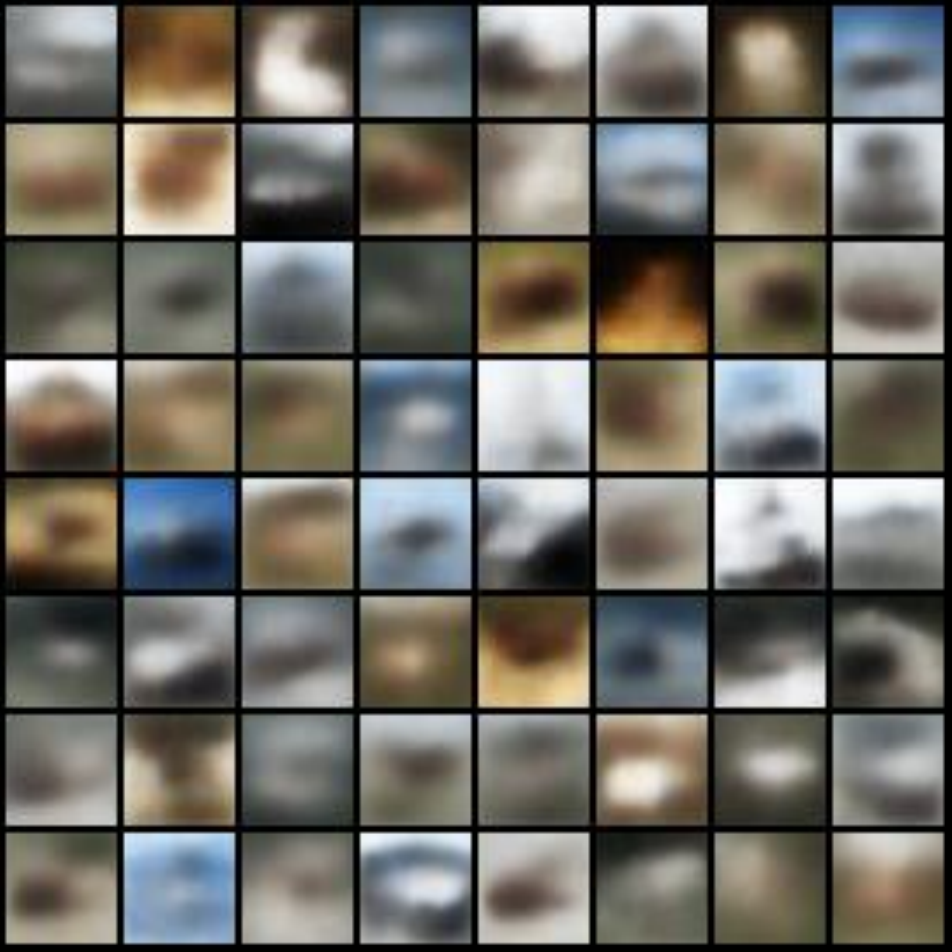}
\end{subfigure}
\caption{Random samples from the Lipshitz VAEs trained on MNIST, Fashion-MNIST, CIFAR-10}
\label{fig:LipschitzVAESamples}
\end{figure*}

First, we compute the complexity estimate $L(\rvx)$ by compressing the input $\rvx$ with JPEG2000. The result represents a string of bits: $C(\rvx)$. After that we apply the normalization of the length of the resulting string by dimensionality $d$:
\begin{equation*}
    L(\rvx) = \frac{\left|C(\rvx)\right|}{d} .
\end{equation*}

Subsequently the input complexity score is calculated in the following way (in bits per dimension):
\begin{equation}
    S(\rvx) = -\log p(\rvx) - L(\rvx)
\end{equation}
The higher the $S$ score, the more indicative it is that the current input is the outlier.

\section{Compactness Ablation}

Since the placement of the outliers within the unconstrained compact space with Vanilla VAEs is basically arbitrary, it can be the case that some outliers will still be successfully detected via hole indicator when these outliers are mapped within the same space as the inliers. Hence, in order to make an appropriate ablation study only for the compactness, we conducted the following experiments. We gradually increase the pixel intensity of the images from black to higher values by multiplying it with a scalar. We calculate the hole indicator for each intensity step for both Lipschitz VAE and Vanilla VAE. The corresponding results can be observed in the Table~\ref{table:CompactnessAblation}.

\begin{table}[h]
\centering
\caption{Ablation of compactness with hole indicator.} 
\label{table:CompactnessAblation}
\resizebox{\columnwidth}{!}{
\begin{tabular}{lcccccccc}
\toprule
& \textbf{1x} & \textbf{3x} & \textbf{5x}  & \textbf{7x} & \textbf{9x} & \textbf{11x} & \textbf{13x} & \textbf{15x} \\
\midrule
\textbf{Vanilla VAE} &  100 &  100 & 100 & 99.69 & 53.40 & 0 & 0 & 0.03\\
\textbf{Lipschitz VAE} & 100 & 100 & 100 & 99.48 &  99.36 & 95.32 & 99.48  & 95.70 \\
\bottomrule
\end{tabular}
}
\end{table}

As can be seen from the obtained values, there is a clear transition from the detectable outliers vs. non-detectable ones through the latent holes in the case of Vanilla VAE, and no degradation of the results in the case with the Lipschitz VAEs.

\section{Samples from Lipschitz VAEs}

Figure~\ref{fig:LipschitzVAESamples} depicts random samples from the Lipschitz VAEs.

\end{document}

%% file: math_commands.tex
\usepackage{amsmath,amsfonts,bm}

\newcommand{\reals}{\mathbb{R}}

\newcommand{\NNgen}{\mathcal{N}_{n,m}^\sigma}

\def\eqref#1{equation~\ref{#1}}

\def\1{\bm{1}}

\def\rvx{{\mathbf{x}}}

\def\rvz{{\mathbf{z}}}

\def\ervx{{\textnormal{x}}}

\DeclareMathAlphabet{\mathsfit}{\encodingdefault}{\sfdefault}{m}{sl}
\SetMathAlphabet{\mathsfit}{bold}{\encodingdefault}{\sfdefault}{bx}{n}

\def\sL{{\mathbb{L}}}

\newcommand{\E}{\mathbb{E}}

\newcommand{\Var}{\mathrm{Var}}

